\documentclass[runningheads]{llncs}

\usepackage[T1]{fontenc}
\usepackage[utf8]{inputenc}
\usepackage{comment}
\usepackage{todonotes}
\usepackage{amsmath}
\usepackage{amssymb}
\usepackage{xspace}
\usepackage{etoolbox}

\newcommand{\st}{S}

\newcommand{\putaway}[1]{}

\newcommand{\atmset}{\ensuremath{\mathit{Atm}}\xspace}

\protected\def\MSO{\ifmmode \mbox{\sc MSO} \else {\sc MSO}\xspace\fi}
\protected\def\FO{\ifmmode \mbox{\sc FO} \else {\sc FO}\xspace\fi}
\protected\def\MMSO{\ifmmode {\mbox{\sc M}\MSO} \else {{\sc M}\MSO}\xspace\fi}
\protected\def\MFO{\ifmmode \mbox{\sc MFO} \else {\sc MFO}\xspace\fi}

\newcommand{\classbelbase}{\mathbf{M} }
\newcommand{\classbelbaseuniv}{\mathbf{M}_{\mathit{univ}} }

\newcommand{\lang}{\ensuremath{\mathcal{L}}}

\newcommand{\langplus}{\ensuremath{\mathcal{L}}^+}

\newcommand{\union}{\cup}

\newcommand{\suchthat}{\mid}

\newcommand{\limply}{\rightarrow}

\newcommand{\bottom}{\bot}
\newcommand{\lequiv}{\leftrightarrow}

\renewcommand{\phi}{\varphi}

\newcommand{\agtset}{\ensuremath{\textit{Agt}}\xspace}

\newcommand{\AGT}{\mathit{Agt}}

\newcommand{\algofunction}{\textbf{function }}

\newcommand{\algoprocedure}{\textbf{procedure }}

\newcommand{\algofor}{\textbf{for }}
\newcommand{\algodo}{\textbf{do }}

\definecolor{algocommentbackgroundcolor}{rgb}{1,1,0.5}

\newcommand{\algowhile}{\textbf{while }}

\newcommand{\algoif}{\textbf{if }}
\newcommand{\algothen}{\textbf{then }}

\newcommand{\algoendif}{\textbf{endIf }}

\newcommand{\algomatch}{\textbf{match }}

\newcommand{\algocase}{\textbf{case }}

\newcommand{\algoreturn}{\textbf{return }}

\newlength{\algoindentlongueur}
\newcommand{\algoindent}{\hspace*{\algoindentlongueur}}

\newlength{\algoindentavantvrulelongueur}
\newcommand{\algoindentavantvrule}{\hspace*{\algoindentavantvrulelongueur}}

\newlength{\dummy}

\newsavebox{\frameminipageboiteavecunnomsuperlongdesortequonnepuissepaslereutiliser}
\newenvironment{frameminipage}[2][c]{%
\begin{lrbox}{\frameminipageboiteavecunnomsuperlongdesortequonnepuissepaslereutiliser}%
\begin{minipage}[#1]{#2}%
} {%
\end{minipage}%
\end{lrbox}%
\framebox{\usebox{\frameminipageboiteavecunnomsuperlongdesortequonnepuissepaslereutiliser}}%
}

\newenvironment{algo} {
  \begin{frameminipage}{\linewidth}
} {
  \end{frameminipage}
}

\newenvironment{algobloc}{\setlength{\dummy}{\linewidth}\addtolength{\dummy}{- \algoindentlongueur}\addtolength{\dummy}{- \algoindentavantvrulelongueur}\algoindentavantvrule\vrule\algoindent\begin{minipage}{\dummy}}{\end{minipage}}

\newenvironment{algoblocfunction}[1]
{\algofunction #1 \\  \begin{algobloc}}
{\end{algobloc}}

\newenvironment{algoblocfor}[1]
{\algofor #1 \algodo \\  \begin{algobloc}}
{\end{algobloc}}

\newenvironment{algoblocmatch}[1]
{\algomatch #1 \algodo \\  \begin{algobloc}}
	{\end{algobloc}}

\tikzstyle{zoneprogramcounter} = [fill=gray!40, draw=none]
\tikzstyle{zonedatacounter} = [fill=gray!40, draw=none, decorate, decoration={snake, amplitude=1pt, segment length=4pt}]
\tikzstyle{portion} = [densely dotted, shade, top color = white, bottom color = gray!40]
\tikzstyle{portiondata} = [densely dotted, shade, top color = white, bottom color = gray!40, decorate, decoration={snake, amplitude=1pt, segment length=4pt}]
\tikzset{
	double arrow/.style args={#1 colored by #2 and #3}{
		-stealth,line width=#1,#2, 
		postaction={draw,-triangle 90 cap,#3,line width=(#1)-4pt,
			shorten <=4pt,shorten >=4}, 
	}
}
\tikzstyle{tapearrow} = [double arrow=10pt colored by black!50!white and black!20!white, -triangle 90 cap, fill = white]%
\tikzstyle{execarrow} = [line width=1pt,->]%

\tikzstyle{world}=[inner sep=0.5mm]
\tikzstyle{event}=[fill=gray!30, inner sep=0.5mm]
\tikzstyle{realworldarrowfromleft} = [initial left, initial text={}]
\tikzstyle{realworldarrowfromright} = [initial right, initial text={}]

\protected\def\DTIME{\ifmmode \mbox{\sc Dtime} \else {\sc Dtime}\xspace\fi}
\protected\def\NTIME{\ifmmode \mbox{\sc Ntime} \else {\sc Ntime}\xspace\fi}
\protected\def\DSPACE{\ifmmode \mbox{\sc Dspace} \else {\sc Dspace}\xspace\fi}
\protected\def\NSPACE{\ifmmode \mbox{\sc Nspace} \else {\sc Nspace}\xspace\fi}
\protected\def\NP{\ifmmode \mbox{\sc NP} \else {\sc NP}\xspace\fi}
\protected\def\coNP{\ifmmode \mbox{\sc coNP} \else {\sc coNP}\xspace\fi}
\protected\def\NPSPACE{\ifmmode \mbox{\sc NPspace} \else {\sc NPspace}\xspace\fi}
\protected\def\PSPACE{\ifmmode \mbox{\sc Pspace} \else {\sc Pspace}\xspace\fi}
\protected\def\EXPSPACE{\ifmmode \mbox{\sc Expspace} \else {\sc Expspace}\xspace\fi}
\protected\def\TWOEXPSPACE{\ifmmode \mbox{\sc 2Expspace} \else {\sc 2Expspace}\xspace\fi}
\protected\def\PTIME{\ifmmode \mbox{\sc P} \else {\sc P}\xspace\fi}
\protected\def\NPTIME{\ifmmode \mbox{\sc NP} \else {\sc NP}\xspace\fi}
\protected\def\EXPTIME{\ifmmode \mbox{\sc Exptime} \else {\sc Exptime}\xspace\fi}
\protected\def\AEXPTIME{\ifmmode \mbox{\sc Aexptime} \else {\sc Aexptime}\xspace\fi}
\protected\def\NEXPTIME{\ifmmode \mbox{\sc NExptime} \else {\sc NExptime}\xspace\fi}
\protected\def\2EXPTIME{\ifmmode \mbox{\sc 2-Exptime} \else {\sc
		2-Exptime}\xspace\fi}
\DeclareRobustCommand{\kEXPTIME}[1][k]{\ifmmode \mbox{\sc $#1$-Exptime}
	\else {\sc $#1$-Exptime}\xspace\fi}
\DeclareRobustCommand{\kNEXPTIME}[1][k]{\ifmmode \mbox{\sc $#1$-NExptime}
	\else {\sc $#1$-NExptime}\xspace\fi}
\DeclareRobustCommand{\kEXPSPACE}[1][k]{\ifmmode \mbox{\sc $#1$-Expspace}
	\else {\sc $#1$-Expspace}\xspace\fi}
\protected\def\ELEMENTARY{\ifmmode \mbox{\sc Elementary} \else {\sc Elementary}\xspace\fi}
\protected\def\AEXPpol{\ifmmode \mbox{{\sc A}_{\text{pol}}\EXPTIME} \else
	{\sc A}$_{\text{pol}}$\EXPTIME\fi}
\protected\def\APTIME{\ifmmode \mbox{\sc Aptime} \else {\sc Aptime}\xspace\fi}
\protected\def\AEXPSPACE{\ifmmode \mbox{\sc Aexpspace} \else {\sc Aexpspace}\xspace\fi}

\newcommand{\gloups}[1]{\bigcirc}

\tikzstyle{cell} = [draw,minimum height=5mm,minimum width=5mm]
\tikzset{
	cellcolor/.cd,
	0/.style={fill=gray!20!white},
	1/.style={fill=yellow!20!white}
}
\tikzstyle{cellalive} = [fill=yellow!20!white]

\newcommand\listeventsempty\epsilon

\definecolor{C}{rgb}{1,1,1}
\definecolor{C }{rgb}{0.8,0.8,0.8}

\definecolor{Ca}{rgb}{0.95,1,0.5}
\definecolor{Cb}{rgb}{1,0.9,0.5}
\definecolor{Cw_2}{rgb}{0.95,1,0.5}
\definecolor{Cw_n}{rgb}{1,0.9,0.5}

\definecolor{Cq'-}{rgb}{0.5,0.5,1}
\definecolor{C-q_0}{rgb}{0.5,0.5,1}
\definecolor{C-q_1}{rgb}{0.5,0.4,1}
\definecolor{C-q'}{rgb}{0.5,0.8,0.5}
\definecolor{Cq_f-}{rgb}{0.4,0.9,0.5}

\definecolor{Cq_0,w_1}{rgb}{1,0.6,0.5}
\definecolor{Cq_0,}{rgb}{1,0.6,0.4}
\definecolor{Cq_2,}{rgb}{0.8,0.9,0.4}
\definecolor{Cq_2,a}{rgb}{0.8,0.9,0.4}
\definecolor{Cq_f,a}{rgb}{0.7,0.9,0.4}
\definecolor{Cq_0,a}{rgb}{1,0.6,0.5}
\definecolor{Cq_0,b}{rgb}{1,0.5,0.5}
\definecolor{Cq_1,a}{rgb}{1,0.5,0.4}
\definecolor{Cq_1,b}{rgb}{1,0.4,0.4}
\definecolor{Cq,a}{rgb}{1,0.5,0.5}
\definecolor{Cq',a}{rgb}{1,0.5,0.5}
\definecolor{Cq',b}{rgb}{1,0.5,0.5}
\definecolor{Cq_0,1}{rgb}{0.5,0.5,0.5}

\newcommand{\expbel}[1] {\triangle_{#1}   }
\newcommand{\impbel}[1] {\Box_{#1}  }

\newcommand*{\impbelmost}[1]{\Box_{#1}^{\small\complement} }

\newcommand{\impbelposs}[1] {\Diamond_{#1}  }
\newcommand{\impbelmostposs}[1] {\Diamond_{#1}^{\small\complement}  }

\newcommand{\impbelonly}[1] {\Box_{#1}^o   }

\newcommand{\bnf}{::=}

\newcommand*{\prop}{p}

\newcommand{\fraglang}{ \mathcal{L}_{0} }

\newcommand{\stateval}{\mathit{V}}

\newcommand{\belbaseset}{\mathit{B}}

\newcommand{\relstate}[1]{\mathcal{R}_{#1}}

\newcommand*{\relstatecomp}[1]{\mathcal{R}_{#1}^\complement}

\newcommand{\iconstraint}{\mathit{Cxt}}

\newcommand{\setbelbase}{\mathbf{S} }

\newcommand{\defin}{~\stackrel{\mbox{\scriptsize def}}{=}~} 

\newcommand{\supervarphi}{\varphi_0}

    \newcommand{\relepist}[1]{\mathcal{R}_{#1}}
        \newcommand{\relepistcomp}[1]{\mathcal{R}_{#1}^\complement}

\providetoggle{long}
\settoggle{long}{true}

\newcommand{\inJELIAversiononly}[1]{\nottoggle{long}{#1}{}}
\newcommand{\inlongversiononly}[1]{\iftoggle{long}{#1}{}}

\newcommand{\setbelbasecustom}{\setbelbase_{\setrelevantformulas}}
\newcommand{\setrelevantformulas}{\Gamma}

\begin{document}
\title{Base-based Model Checking\\ for Multi-Agent Only Believing\inlongversiononly{\\(long version)}%
    \thanks{%
        This work is partially supported by the project epiRL (“Epistemic Reinforcement Learning”) ANR-22-CE23-0029,
        the project CoPains (“Cognitive Planning in Persuasive Multimodal Communication”) ANR-18-CE33-0012
        and
        the AI Chair project Responsible AI (ANR-19-CHIA-0008) both from the French National Agency of Research.
        Support from the Natural Intelligence Toulouse Institute (ANITI) is also gratefully acknowledged.%
    }%
}
%
\author{Tiago de Lima\inst{1}
    \and Emiliano Lorini\inst{2}
    \and Fran\c{c}ois Schwarzentruber\inst{3}
}
\authorrunning{T. de Lima et al.}
%
\institute{CRIL, Univ Artois and CNRS, Lens, France, \email{delima@cril.fr}
      \and IRIT, CNRS, Toulouse University, France, \email{lorini@irit.fr}
      \and ENS Rennes, France, \email{francois.schwarzentruber@ens-rennes.fr}}
    
\maketitle

\begin{abstract}
\iftoggle{long}{%
We present a novel semantics for the language of multi-agent only believing exploiting belief bases,
and show how to use it for automatically checking formulas of this language
and of its dynamic extension with private belief expansion operators. 
We provide a PSPACE algorithm for model checking relying  on a reduction to QBF
and alternative dedicated algorithm relying on the exploration of the state space.
We present an implementation of the QBF-based algorithm and some experimental results on computation time in a concrete example.
}{
We present a novel semantics for the language of multi-agent only believing exploiting belief bases,
and show how to use it for automatically checking formulas of this language.
We provide a PSPACE algorithm for model checking relying on a reduction to QBF,
an implementation and some experimental results on computation time in a concrete example.
} 
\end{abstract}
\section{Introduction}

Using belief bases for building a semantics for epistemic logic was initially proposed by Lorini \cite{LoriniAAAI2018,LoriniAI2020}.
In \cite{DBLP:journals/corr/abs-1907-09114} it was shown that such a semantics allows to represent the concept of universal epistemic model which is tightly connected with the concept of universal type space studied by game theorists \cite{UnivZamir}.
A qualitative version of the universal type space with no probabilities involved is defined by Fagin et al. \cite{FaginUniv2} (see also \cite{FaginUniv1}).
Broadly speaking, a universal epistemic model for a given situation is the most general model which is compatible with that situation.
It is the model which only contains information about the situation and makes no further assumption.
From an epistemic point view, it can be seen as the model with maximal ignorance with respect to the description of the situation at stake.

Such a universal epistemic model has been shown to be crucial for defining a proper semantics for the concept of multi-agent only knowing (or believing)  \cite{DBLP:conf/ijcai/Lakemeyer93,DBLP:conf/aaai/Halpern93}, as a generalization of the concept of single-agent only knowing (or believing) \cite{DBLP:journals/ai/Levesque90}.\footnote{%
    As usual, the difference between knowledge and belief lies in the fact that the former is always correct while the latter can be incorrect.%
}
However, the construction of this semantics is far from being straightforward.
Halpern \& Lakemeyer \cite{DBLP:journals/logcom/HalpernL01} use the proof-theoretic notion  of canonical model for defining it.
The limitation of the canonical model is its being infinite thereby not being exploitable in practice.
In a more recent work, Belle \& Lakemeyer \cite{DBLP:conf/kr/BelleL10} provided an inductive proof-independent definition of the semantics for multi-agent only knowing  which departs from the standard semantics of multi-agent epistemic logic based on multi-relational Kripke structures. 
Finally, Aucher \& Belle  \cite{DBLP:conf/ijcai/AucherB15} have shown how to interpret  the language of multi-agent only knowing on standard Kripke structures. 
Although being independent from the proof theory, these last two accounts are fairly non-standard or quite involved.
They rely either on an inductive definition (Belle \& Lakemeyer) or on a complex syntactic representation up to certain modal depth (Aucher \& Belle) of the  multi-agent epistemic structure used for interpreting the multi-agent only knowing language.
 
In this paper, we concentrate on the logic of multi-agent only believing based on the logic K for beliefs.
We show how to use the belief base semantics and its construction of the universal model to automatically check formulas of the corresponding language.
The novel contribution of the paper is twofold:
\begin{itemize}
    \item Although the idea of using belief bases as a semantics for epistemic logic has been proposed  in previous work, this is the first attempt to use them in     the context of the logic of multi-agent only believing\inlongversiononly{ and of its extension with private belief expansion operators}.
    \item Moreover, we are the first to provide a model checking algorithm for the logic of multi-agent  only believing, to implement it and to test it                 experimentally on a concrete example. The belief base semantics helped us to accomplish this task given its compactness and manageability. 
\end{itemize}

\paragraph{Outline.}
In Section~\ref{sec:langsem}, we first recall the belief base semantics introduced in our previous work \cite{LoriniAAAI2018,LoriniAI2020}.
We  show how to interpret the language of multi-agent only believing and how to define the universal model in it. 
\inlongversiononly{%
In Section~\ref{sec:example}, we introduce an example to illustrate the framework.%
}
In Section~\ref{sec:mc}, we move to model checking formulated in the belief base semantics.
\inlongversiononly{%
We provide a PSPACE algorithm for model checking relying on a reduction to QBF.%
}
In Section~\ref{sec:expres}, we present an implementation of the QBF-based algorithm and some experimental results on computation time in the example.
\inlongversiononly{%
In Section~\ref{sec:dynamic} we propose an extension of the setting with private belief expansion operators, and demonstrate that the model checking problem remains in PSPACE.%
}
Section~\ref{sec:conclusion} concludes.%
\inJELIAversiononly{\footnote{The extended version of this paper, including proofs and examples, is available at ArXiv \url{}.}}

\section{Language and semantics}%
\label{sec:langsem}

The multi-agent epistemic language introduced in \cite{LoriniAI2020} has two basic epistemic modalities:
one for explicit belief,
and another one for implicit belief. 
An agent's explicit belief corresponds to a piece of information in the agent's belief base.
An agent's implicit belief corresponds to a piece of information that is derivable from the agent's explicit beliefs. 
In other words, if an agent can derive $\varphi$ from its explicit beliefs, it implicitly believes \emph{at least} that $\varphi$ is true.
We consider the extension of this epistemic language by complementary modalities for implicitly believing \emph{at most}.
The  \emph{at least} and \emph{at most} modalities can be combined to represent the concept of \emph{only} believing.

The semantics over which the language is interpreted exploits belief bases.
Unlike the standard multi-relational Kripke semantics for epistemic logic in which the agents' epistemic accessibility relations over possible worlds are given as primitive, in this semantics they are computed from the agents' belief bases. 
 Specifically, in this semantics it is assumed that at state $S$ an agent considers a state $S'$ possible
 (or state $S'$ is epistemically accessible to the agent at state $S$)
if and only if $S'$ satisfies all formulas included in the agent's  belief base at $S$. 
This idea of computing the agents' accessibility relations from the state description is shared with the semantics of epistemic logic based on interpreted systems  \cite{Fagin1995,LomuscioRaimondi2015}.
However, there is an important difference.
While the interpreted system semantics relies on the abstract notion of an agent's local state, in the belief base semantics an agent's local state is identified with its concrete belief base. 

\subsection{Semantics}

Assume
a countably infinite set of atomic propositions $\atmset = \{p,q, \ldots \}$
and
a finite set of agents $\agtset = \{ 1, \ldots, n \}$.
We define the language $\fraglang$ for
explicit belief 
by the following grammar in Backus-Naur Form (BNF):
\[ \fraglang\defin \alpha \bnf \prop
                          \mid \lnot\alpha
                          \mid \alpha \land \alpha
                          \mid \expbel{i}\alpha, \]
where $p$ ranges over $\atmset$
and $i$ ranges over $\agtset$.
$\fraglang$ is the language used to represent explicit beliefs.
The formula $\expbel{i} \alpha$ reads ``agent $i$ has the explicit belief that $\alpha$''.
In our semantics, a state is not a primitive notion but it is decomposed into different elements:
one belief base per agent
and an interpretation of propositional atoms. 

\begin{definition}[State]%
\label{state}
A state is a
tuple $\st = \big( (\belbaseset_i)_{i \in \agtset},\stateval \big)$
where $\belbaseset_i \subseteq \fraglang$
is agent $i$'s belief base, 
and 
$\stateval \subseteq \atmset$ is the actual environment.
The set of all states is noted $\setbelbase$.
\end{definition}

The following definition specifies truth conditions for formulas in $\fraglang$.

\begin{definition}[Satisfaction relation]%
\label{satrel}
Let 
$\st = \big( (\belbaseset_i)_{i \in \AGT}, \stateval \big) \in \setbelbase$.
Then,
\begin{eqnarray*}
    \st \models \prop & \Longleftrightarrow & \prop \in \stateval,\\
    \st \models \lnot\alpha & \Longleftrightarrow & \st \not\models \alpha,\\
    \st \models \alpha_1 \land \alpha_2 & \Longleftrightarrow & \st \models \alpha_1 \text{ and } \st \models \alpha_2,\\
    \st \models \expbel{i} \alpha & \Longleftrightarrow & \alpha \in \belbaseset_i.
\end{eqnarray*}
\end{definition}

Observe in particular the set-theoretic interpretation of the explicit belief operator in the previous definition:
agent $i$ has the explicit belief that $\alpha$
if and only if $\alpha$ is included in its  belief base.

The following definition introduces the agents' epistemic relations.
They are computed  from the agents' belief bases. 
\begin{definition}[Epistemic relation]%
\label{DefAlternative}
Let $i \in \AGT$.
Then,
$\relstate{i}$ is the binary relation on $\setbelbase$
such that,
for all
$\st = \big( (\belbaseset_i)_{i \in \AGT}, \stateval\big)
, 
\st' = \big( (\belbaseset_i')_{i \in \AGT}, \stateval'\big)
\in \setbelbase$,
we have
$\st \relstate{i} \st' \text{ if and only if } \forall \alpha \in \belbaseset_i : \st' \models \alpha$.
\end{definition} 

$\st \relstate{i} \st'$ means that $\st'$ is an epistemic alternative for agent $i$ at $\st$,
that is to say,
$\st'$ is a state that at $\st$ agent $i$ considers possible.
The idea of the previous definition is that $\st'$ is an epistemic alternative for agent $i$ at $\st$
if and only if,
$\st'$ satisfies all facts that agent $i$ explicitly believes at $\st$.
 
The following definition introduces the concept of model, namely a state supplemented with a set of states, called \emph{context}. 
The latter includes all states that are compatible with the agents' common ground,
i.e., the body of information that the agents commonly believe to be the case \cite{StalnakerCommonGround}. 

\begin{definition}[Model]%
\label{univM}
A model is a pair $(\st,\iconstraint)$
with $\st \in \setbelbase$
and $\iconstraint \subseteq \setbelbase$. 
The class of models is noted $\classbelbase$.
\end{definition}

Note that in a model $(\st,\iconstraint)$, the state $\st$ is not necessarily an element of the context $\iconstraint$ due to the fact that we model belief instead of knowledge.
Therefore, the agents' common ground represented by the context $\iconstraint$ may be incorrect  and not include the actual state. 
If we modeled knowledge instead of belief, we would have to suppose that $\st \in \iconstraint$. 

Let $\setrelevantformulas = (\setrelevantformulas_i)_{i \in \AGT}$
where, for every $i\in \AGT$, $\setrelevantformulas_i$
represents agent $i$'s vocabulary.
A $\setrelevantformulas$-universal model is a model containing all states at which an agent $i$'s explicit beliefs are built from its vocabulary $\setrelevantformulas_i$. 
In other words, an agent's vocabulary plays a role analogous to that of the notion of awareness in the formal semantics  of awareness \cite{Fag87}. 
The notion of $\setrelevantformulas$-universal model is defined as follows. 

\begin{definition}[Universal model]%
\label{univM2}
The model $(\st,\iconstraint)$
in $\classbelbase$
is said to be $\setrelevantformulas$-universal
if $\st \in \iconstraint=\setbelbasecustom$,
with
$
\setbelbasecustom=\Big\{ 
  \big( (\belbaseset_i')_{i \in \agtset},\stateval'\big)
\in \setbelbase\suchthat \forall i \in \agtset, 
\belbaseset_i' \subseteq \setrelevantformulas_i
\Big\}.
$
The class of $\setrelevantformulas$-universal models
is noted $\classbelbaseuniv(\setrelevantformulas)$.
\end{definition}

$\setrelevantformulas = (\setrelevantformulas_i)_{i \in \AGT}$ is also called agent vocabulary profile.
Clearly, when $\setrelevantformulas = \fraglang^n$, we have  $\setbelbasecustom=\setbelbase$. 
A model  $(\st,\setbelbase)$ in $\classbelbaseuniv(\fraglang^n)$ is a model with maximal ignorance:
it only  contains the information provided by the actual  state $\st$. 
For simplicity, we write $\classbelbaseuniv$ instead of $\classbelbaseuniv(\fraglang^n)$. 

\subsection{Language}

In this section,
we introduce a language for implicitly believing \emph{at most} and implicitly believing \emph{at least} on the top of the language $\fraglang$ defined above.
It is noted  $\lang$ and defined by:
\[ \lang \defin \phi \bnf \alpha
                     \mid \neg\phi
                     \mid \phi\land\phi
                     \mid  \impbel{i}\phi
                     \mid \impbelmost{i}\phi, \]
where $\alpha$ ranges over $\fraglang$
and $i $ ranges over $\AGT $.
The other Boolean constructions $\top$, $\bot$, $\lor$, $\limply$, $\oplus$, and $\lequiv$ are defined from $\alpha$, $\neg$ and $\land$ in the standard way.

The formula $ \impbel{i}\phi$ is read
``agent $i$ \emph{at least} implicitly believes that $\varphi$'',
while $ \impbelmost{i}\phi$ is read
``agent $i$ \emph{at most} implicitly believes that $\neg \varphi$''. 
Alternative readings of formulas
$\impbel{i}\phi$ and $\impbelmost{i}\phi$ are, respectively,
``$\phi$ is true at all states that agent $i$ considers possible''
and
``$\phi$ is true at all states that agent $i$ does not consider possible''.
The latter is in line with the reading of the normal modality and the corresponding ``window'' modality in the context of Boolean modal logics \cite{Gargov1990}.
The duals of the operators $\impbel{i}$ and $\impbelmost{i}$ are defined in the usual way, as follows:
$\impbelposs{i}\phi \defin \neg\impbel{i}\neg\phi$
and
$\impbelmostposs{i}\phi \defin \lnot\impbelmost{i}\lnot\phi$.
Formulas in the language $\lang$ are interpreted relative to a model $(\st,\iconstraint)$.
(Boolean cases are omitted since they are defined as usual.) 
  
\begin{definition}[Satisfaction relation (cont.)]%
\label{truthcond2}
Let $(\st,\iconstraint) \in \classbelbase$. Then:
\begin{eqnarray*}
    (\st,\iconstraint) \models \alpha & \Longleftrightarrow & \st \models \alpha,\\
    (\st,\iconstraint)\models \impbel{i}\phi & \Longleftrightarrow
        & \forall \st' \in  \iconstraint : \text{if } \st \relstate{i} \st'
            \text{then } (\st',\iconstraint) \models \phi,\\
    (\st,\iconstraint) \models \impbelmost{i}\phi & \Longleftrightarrow
        & \forall \st' \in \iconstraint : \text{if } \st \relstatecomp{i} \st'
            \text{then } (\st',\iconstraint) \models \phi,
\end{eqnarray*}
with $\relstatecomp{i} = (\setbelbase \times \setbelbase) \setminus \relstate{i}$. 
\end{definition}

Note that $\st \relstatecomp{i} \st'$ just means that at state $\st$ agent $i$ does not consider state $\st'$ possible.
Moreover, interpretations of the two modalities $\impbel{i}$ and $ \impbelmost{i}$ are restricted to the actual context $\iconstraint$. 
The only believing modality ($\impbelonly{i}$)
is defined as follows: 
\begin{align*}
    \impbelonly{i}\varphi & \defin \impbel{i}\phi \land \impbelmost{i} \lnot\phi. 
\end{align*}
    
Notions of satisfiability and validity of $\lang$-formulas for the class of models $\classbelbase$ are defined in the usual way:
$\phi$ is satisfiable if there exists $(\st,\iconstraint) \in \classbelbase$ such that $(\st,\iconstraint) \models \phi$,
and
$\varphi$ is valid if $\lnot\phi$ is not satisfiable. 
    
\inlongversiononly{%
In \cite[Theorem 26]{DBLP:journals/corr/abs-1907-09114}, it is shown that when restricting to the fragment of the language $\lang$ where formulas containing explicit beliefs are disallowed
(i.e., in the definition of $\lang$ $\alpha $ is replaced by $p$), 
the set of satisfiable formulas relative to the class $\classbelbaseuniv$ is the same as the set of satisfiable formulas relative to the class of qualitative universal type spaces,
as defined in \cite{FaginUniv2}.
The latter is similar to the class of $k$-structures, as defined by Belle \& Lakemeyer \cite{DBLP:conf/kr/BelleL10}.%
\footnote{%
    Although it has not been formally proven, we believe that Belle \& Lakemeyer's semantics and Fagin et al.'s semantics are nothing but different formulations of the same class of universal epistemic structures.%
}
In Section \ref{sec:mc}, we will show that $\setrelevantformulas$-universal models of Definition \ref{univM2} provide an adequate and compact semantics for model checking formulas of the language $\lang$. 
But, before delving into model checking, we illustrate our language with the help of an example.%
} 
    
\inlongversiononly{%
\section{Example}%
\label{sec:example}

We give two variants of the example, the first
focused on first-order beliefs and the second
focused on second-order beliefs.

\begin{example}%
\label{eg:committee}
Agents in $\AGT$ are members of a selection committee
for an associate professor position. They have to choose which candidates to admit to the second round of selection  consisting  in an interview. Committee members  and candidates
work in the same scientific community. 
Therefore, it is possible that they  co-authored some papers in the past.
Assume there are $m$ candidates $\mathit{Cand}=\{c_1, \ldots, c_m\}$.
In order to formalize the example we use atomic propositions of the
form $\mathsf{vote}(i{,}c )$,
with $i \in \AGT$
and $c \in \mathit{Cand}$,
standing for
 ``agent $i$ votes for candidate $c$''.
 The first two rules
 of the game
 state that each committee member must vote for exactly
 one candidate (at least one candidate and no more than one):
\begin{align*}
    & \alpha_1 \defin \bigwedge_{i \in \AGT}%
        \bigvee_{c \in \mathit{Cand}}\mathsf{vote}(i{,}c),\\
    & \alpha_2 \defin \bigwedge_{i \in \AGT}\bigwedge_{c,c' \in \mathit{Cand}, c\neq c'}%
        \big(  \mathsf{vote}(i{,}c) \rightarrow \neg \mathsf{vote}(i{,}c') \big).
\end{align*}
The third rule states
that a member of the committee cannot vote for a candidate with whom she/he co-authored an article in the past:
 \begin{align*}
    & \alpha_3 \defin \bigwedge_{i \in \AGT}
        \bigwedge_{c \in \mathbf{f}(i)} \neg\mathsf{vote}(i{,}c),
\end{align*}
where $\mathbf{f}: \AGT \longrightarrow 2^{\mathit{Cand}}$
is a function mapping each member 
of the committee to her/his co-authors.
A candidate $c$ is admitted to the interview if and only if at least one member of the committee has voted for her/him. This is expressed by the
following abbreviation:
\begin{align*}
    & \mathsf{adm}(c )\defin  \bigvee_{i \in \AGT}\mathsf{vote}(i{,}c ) .
\end{align*}
Let us consider the variant of the example
in which
the evaluation committee and
the set of candidates 
have the same cardinality
and a committee member co-authored an article with only her/his matching candidate in the linear order. That is, we suppose:
\begin{align*}
  & |\AGT|=|\mathit{Cand}|>2, \text{ and }\\
  & \forall i \in \AGT ,
    \mathbf{f}(i)=\{c_i\}.
\end{align*}
Furthermore,
we suppose that
(i)
each  committee member
except the last one votes
for her/his next candidate 
in the linear order,
while
the last committee member $n$ votes for her/his previous
candidate $n-1$; 
(ii) the vote by a committee member is secret
(i.e., a committee member only has epistemic  access to 
her/his vote),
(iii) all committee members know the results
of the selection, namely, which candidates are admitted to the interview and which are not. The three hypotheses (i), (ii) and (iii) are fully expressed by the state 
$\st_0 = \big( (\belbaseset_i)_{i \in \AGT}, \stateval\big)$
such that,
for every $1 \leq i < n$, 
\begin{align*}
    \belbaseset_i= & \big\{
        \mathsf{vote}(i{,}c_{i+1}),
        \neg\mathsf{adm}(c_1),
        \mathsf{adm}(c_2), \ldots , \mathsf{adm}(c_n),
        \alpha_1, \alpha_2, \alpha_3
    \big\},
\end{align*}
and moreover, 
\begin{align*}
    \belbaseset_n= & \big\{ 
        \mathsf{vote}(n{,}c_{n-1}),
        \neg\mathsf{adm}(c_1), 
           \mathsf{adm}(c_2),
        \ldots , \mathsf{adm}(c_n),
        \alpha_1, \alpha_2, \alpha_3 
    \big\},\\
    \stateval = & \big\{
        \mathsf{vote}(1{,}c_{2}),
        \ldots, 
        \mathsf{vote}(n-1{,}c_{n}) , \mathsf{vote}(n{,}c_{n-1})
    \big\}.
\end{align*}
The following holds
when 
$|\AGT|=|\mathit{Cand}|=3$:
\begin{equation*}
    (\st_0,\setbelbasecustom) \models  \varphi_0,
\end{equation*}
with
\begin{align*}
    \varphi_0 \defin & \impbelonly{1}\psi_1
                 \land \bigwedge_{i \in  \{2,3\}} \impbelonly{i}\psi_2,
    \\
    \psi_1 \defin &
    \begin{aligned}[t]
        & 
            \mathsf{vote}(1{,}c_{2}) 
            \land \neg\mathsf{vote}(1{,}c_{1})
            \land \neg\mathsf{vote}(1{,}c_{3})
        \\
        & \land 
            \mathsf{vote}(2{,}c_{3})
            \land \neg\mathsf{vote}(2{,}c_{1})
            \land \neg\mathsf{vote}(2{,}c_{2})
        \\
        & \land 
            \mathsf{vote}(3{,}c_{2}) 
            \land \neg\mathsf{vote}(3{,}c_{1})
            \land \neg\mathsf{vote}(3{,}c_{3}),
    \end{aligned}\\
    \psi_2 \defin &
    \begin{aligned}[t]
        & 
            \neg\mathsf{vote}(1{,}c_{1}) 
            \land
            (\mathsf{vote}(1{,}c_{2} ) \oplus \mathsf{vote}(1{,}c_{3})) 
        \\
        & \land 
            \mathsf{vote}(2{,}c_{3} ) 
            \land \neg\mathsf{vote}(2{,}c_{1})
            \land \neg\mathsf{vote}(2{,}c_{2})
        \\
        & \land 
            \mathsf{vote}(3{,}c_{2}) 
            \land \neg\mathsf{vote}(3{,}c_{1})
            \land \neg\mathsf{vote}(3{,}c_{3}).
    \end{aligned}
\end{align*}
and
$\setrelevantformulas_i = \belbaseset_i \union \neg\belbaseset_i$
for every $i \in \AGT$,
(where $\neg\belbaseset_i$ is the set of negations of the formulas in $\belbaseset_i$).
$(\st_0,\setbelbasecustom)$
so defined 
is nothing but a  $\Gamma$-universal
model
in which the agents' vocabularies include all and only those
formulas
in their actual belief bases
and their negations.

This means that,
in the three-agent case,
agent $1$
only 
knows for whom an agent voted 
and for whom 
she/he did not vote,
while 
agent $2$
and agent $3$
only
know
for whom they  voted 
and for whom 
they did not vote, 
and
that agent $1$ voted either for $2$
or for $3$.
Therefore, $2$
and $3$
do not know for whom
$1$
voted. 
Interestingly, 
when 
$|\AGT|=|\mathit{Cand}|>3$:
\begin{align*}
   (\st_0,\setbelbasecustom) \not  \models  \varphi_0.
\end{align*}
\end{example}
   
\newcommand*{\adm}{\mathsf{adm}}

\begin{example}
\label{eg:committee-variant}
It is worth to consider a variant of Example \ref{eg:committee} in which agent $1$ has higher-order explicit beliefs
(i.e., explicit beliefs about other agents' explicit beliefs).
Specifically, we consider a state 
$\st_0' = \big( (\belbaseset_i')_{i \in \AGT}, \stateval' \big)$
such that,
\begin{align*}
    \belbaseset'_1 = & \belbaseset_1
    \cup
    \{ 
        \expbel{2}\lnot\adm(c_1),
        \expbel{2}\adm(c_2),
        \ldots,
        \expbel{2}\adm(c_n),
        \expbel{2}\alpha_1,
        \expbel{2}\alpha_2,
        \expbel{2}\alpha_3
    \}
    \intertext{ and, for every $1 < i \leq n$: }
    \belbaseset_i' = & \belbaseset_i,\\
    \stateval' = & \stateval,
\end{align*}
where $\belbaseset_i$
and $ \stateval$ are defined as above. 
In other words,
committee member $1$ explicitly knows that committee member $2$ explicitly knows the rules of the game as well as the results of the selection.
  
Interestingly, 
when 
$|\AGT|=|\mathit{Cand}|=3$,
the following holds:
\begin{equation*}
    (\st'_0,\setbelbase) \models \impbelonly{2}\psi_2
                           \land \impbel{1}\impbel{2}\psi_2
                           \land \lnot\impbel{1}\impbelonly{2}\psi_2.
\end{equation*}

In words, in the three-agent case, at $\st_0'$,
committee member $2$ only knows that $\psi_2$,
committee member $1$ knows that $2$ knows $\psi_2$,
but $1$ does not know that $2$ only knows that $\psi_2$.
\end{example}
} 

\section{Model checking}%
\label{sec:mc}

The model checking problem is defined in our framework as follows:
\begin{description}
\item[input:]
    an agent vocabulary profile $\setrelevantformulas = (\setrelevantformulas_i)_{i \in \AGT}$ with $\setrelevantformulas_i$ finite for every $i\in \AGT$,
    a finite state $\st_0$ in $\setbelbasecustom$,
    and
    a formula $\phi_0 \in \lang$; 
\item[output:]
    yes if $(\st_0,\setbelbasecustom) \models \phi_0$;
    no otherwise.
\end{description}

\begin{remark}
We suppose w.l.o.g.\ that outer most subformulas of $\supervarphi$ of the form $\triangle_i\alpha$ are such that $\alpha \in \setrelevantformulas_i$.
If this is not the case for some subformulas $\triangle_i \alpha$, then the subformula $\triangle_i\alpha$ will be false anyway and can be replaced by $\bottom$.
\end{remark}

\inlongversiononly{%
\subsubsection{Direct PSPACE algorithm}

\newcommand\mc[3]{mc(#1, #2, #3)}

Figure \ref{figure:genericmc} shows an algorithm $\mc{\st}{\setrelevantformulas}{\phi}$ that checks whether $(\st,\setbelbasecustom) \models \phi$. Note that $\setbelbasecustom$ is not computed explicitly, but implicitly represented by $\setrelevantformulas$. In the algorithm, states $\st$ are represented as vectors of bits indicating for all $i$, for each element $\alpha$ of $\setrelevantformulas_i$ whether $\alpha$ belongs to $i$'s base in $\st$ or not. It also encodes the valuation over atomic propositions appearing in $\setrelevantformulas$ and $\supervarphi$. The states $\st$ manipulated by the algorithm are of size polynomial in the size of the input. The loop \algofor all $\st' \in \setbelbase_0$ such that $\st \relstate i \st'$ works as follows. We consider all the vectors $\st'$. Each such vector $\st'$ represents a state in $\setbelbasecustom$. For each $\st'$ we check in polynomial time whether $\st \relstate i \st'$.

\begin{figure}[ht]
    \begin{center}
 			\begin{algo}
 				\begin{algoblocfunction}{$\mc{\st}{\setrelevantformulas}{\phi}$}
 					
 					\begin{algoblocmatch}{$\phi$}
 						
 						\algocase $p$: \algoreturn $p \in \stateval$ 
 						
 						\algocase $\expbel i \alpha$: \algoreturn $\alpha \in \belbaseset_i$
 						
 						\algocase $\lnot \psi$: \algoreturn not $\mc \st \setrelevantformulas \psi$
 						
 						\algocase $\psi_1 \land \psi_2$: \algoreturn $\mc \st \setrelevantformulas {\psi_1}$ and $\mc \st \setrelevantformulas {\psi_2}$


 						\algocase $\impbel i \psi$:
 						
 						\begin{algobloc}
 							
 							\begin{algoblocfor}{all $\st'\in
 									\setbelbasecustom$ such that $\st \relstate i \st'$} 
 								\algoif not $\mc {\st'} \setrelevantformulas \psi$	\algoreturn false

 							\end{algoblocfor}
 							
 							\algoreturn true
 							
 						\end{algobloc}
 						
 						\algocase $\impbelmost i \psi$:

 						\begin{algobloc}
 							
 							\begin{algoblocfor}{all $\st' \in
 									\setbelbasecustom$ such that $\st \relstatecomp i \st'$} 
 								\algoif not $\mc {\st'} \setrelevantformulas \psi$	\algoreturn false			
 							\end{algoblocfor}
 							\algoreturn true
 						\end{algobloc}
 					\end{algoblocmatch}
 				\end{algoblocfunction}
 			\end{algo}
    \end{center}
\caption{\label{figure:genericmc}Generic algorithm for model checking $\lang$-formulas.}
\end{figure}

\subsubsection{Reduction to TQBF}\label{sec:redTQBF}
} 

\newcommand{\propfortriangle}[2]{x_{#1, #2}}
\newcommand{\setQBFvariableslevel}[1]{X_{#1}}

We propose a reduction to TQBF (true quantified binary formulas).
We introduce TQBF propositional variables $\propfortriangle \alpha k$ for all $\alpha \in \fraglang$ and for all integers $k$.
The variables indexed by $k$ are said to be of level~$k$.
\inlongversiononly{They correspond to the recursive nesting in the procedure $mc$ described in Figure~\ref{figure:genericmc} for the cases $\impbel{i}\phi$ and $\impbelmost{i}\phi$.}
For instance, $\propfortriangle \alpha k$ is true if $\alpha$ is true at some state at depth $k$.
Let $\setQBFvariableslevel {k}$ be the set of formulas of level $k$.
More precisely, $\setQBFvariableslevel {k}$ contains exactly formulas $\propfortriangle {\expbel i \alpha} k$ with $\alpha \in \setrelevantformulas_i$ for any agent $i$, and $\propfortriangle p k$ with $p$ appearing in $\setrelevantformulas$ or~$\supervarphi$.

\newcommand{\QBFrel}{R}

\begin{definition}
We define the function $tr$ that maps any formula of $\lang$ to a QBF-formula by $tr(\supervarphi) := tr_0(\supervarphi)$ with:
\begin{itemize}
    \item $tr_k(\prop) = \propfortriangle \prop k$
    \item $tr_k(\lnot\phi) = \lnot tr_k(\phi)$
    \item $tr_k(\phi \land \psi) = tr_k(\phi) \land tr_k(\psi)$
    \item $tr_k(\expbel{i} \alpha) = \propfortriangle{\expbel{i} \alpha}{k}$
    \item $tr_k(\impbel{i}\phi) = \forall \setQBFvariableslevel{k+1}(\QBFrel_{i,k} \limply tr_{k+1}(\phi))$
    \item $tr_k(\impbelmost{i}\phi) = \forall \setQBFvariableslevel{k+1}(\lnot\QBFrel_{i,k} \limply tr_{k+1}(\phi))$
\end{itemize}
where:
\[ \QBFrel_{i,k} := \bigwedge_{\alpha \in \setrelevantformulas_i} \propfortriangle{\expbel i \alpha}k  \limply tr_{k+1}(\alpha). \]
\end{definition}

\inlongversiononly{The translation $tr_{k}(\impbel{i}\phi)$ corresponds to case $\impbel i \phi$ in the algorithm.}
State $S$ (resp. $S'$) is represented by the truth values of variables in $X_k$ (resp. $X_{k+1}$).
Formula $\QBFrel_{i,k}$ reformulates  $\st \relstate i \st'$.

\newcommand{\descr}[2]{\operatorname{desc}_{#1}(#2)}

\begin{proposition}%
\label{prop:tqbftrans}
Let $\varphi_0 \in \lang$
and $\st_0 = ((B_i)_{i \in \AGT}, V)$.
The following two statements are equivalent:
\begin{itemize}
    \item $(\st_0,\setbelbasecustom) \models \phi_0$ 
    \item
        $\exists \setQBFvariableslevel0 (
            \descr{\st_0}{\setQBFvariableslevel0} \land tr_0(\phi_0)
        )$ is QBF-true,
\end{itemize}
where:
\[ \descr{\st_0}{\setQBFvariableslevel0} :=
         \bigwedge_{i\in\AGT} \left(
            \bigwedge_{\alpha \in B_i} \propfortriangle{\expbel i \alpha}0
            \land
            \bigwedge_{\alpha \in \setrelevantformulas_i \setminus B_i} \lnot 
                \propfortriangle{\expbel i \alpha}0
        \right)
        \land
        \bigwedge_{p \in V} \propfortriangle p 0
        \land
        \bigwedge_{p \not \in V} \lnot \propfortriangle p 0. \]
\end{proposition}

\inlongversiononly{%
\begin{proof}
\newcommand{\valuationdescr}[2]{\operatorname{val}_{#1}(#2)}
Let $\valuationdescr{\st}{\setQBFvariableslevel k}$
represent the unique valuation on $X_k$ satisfying $\descr{\st}{\setQBFvariableslevel k}$.
We prove by induction on the structure of $\phi$ that
$%
    (\st,\setbelbasecustom) \models \phi
    \text{ iff }
    \valuationdescr{\st}{\setQBFvariableslevel k} \models tr_k(\phi)
$,
for all $k$.

Induction base.
Let $\phi = p$, for some $p \in \atmset$.
We have
$(\st,\setbelbasecustom) \models p$
iff
$p \in V$
iff    
$x_{k, p} \in \valuationdescr{\st}{\setQBFvariableslevel k}$
iff
$\valuationdescr{\st}{\setQBFvariableslevel k} \models tr_k(p)$.

Induction step.
The cases for operators $\lnot$ and $\land$ are straightforward.
We proceed with the modal operators in the language:
\begin{itemize}
\item Let $\phi = \expbel i \alpha$.
    We have:
    $(\st,\setbelbasecustom) \models \expbel i \alpha$
    iff
    $\alpha \in B_i$
    iff
    $x_{k, \expbel i \alpha} \in \valuationdescr{\st}{\setQBFvariableslevel k}$
    iff
    $\valuationdescr{\st}{\setQBFvariableslevel k} \models tr_k(\expbel i \alpha)$.
    
\item
    Let $\phi = \impbel i \psi$.
    We denote by $\valuationdescr{\st}{\setQBFvariableslevel {k}} + \valuationdescr{\st'}{\setQBFvariableslevel {k+1}}$ the valuation obtained by concatenating the valuation $\valuationdescr{\st}{\setQBFvariableslevel {k}}$ and $\valuationdescr{\st'}{\setQBFvariableslevel {k+1}}$ (we take the truth values of propositions in $\setQBFvariableslevel {k}$ from the former and the truth values of propositions in $\setQBFvariableslevel {k+1}$ from the latter).
    We have:
\begin{align*}
    (\st,\setbelbasecustom) \models  \impbel i \psi
    & \iff
    \text{ for all $\st' \in \setbelbasecustom$, $\st \relepist{i} \st'$ implies } (\st',\setbelbasecustom) \models  \psi
    \\
    & \iff
    \text{ for all $\st' \in \setbelbasecustom$, $\st \relepist{i} \st'$ implies } \valuationdescr{\st'}{\setQBFvariableslevel {k+1}} \models tr_{k+1}(\psi)
    \\
    & \iff
    \text{ for all $\st' \in \setbelbasecustom$}
    \valuationdescr{\st}{\setQBFvariableslevel {k}} + \valuationdescr{\st'}{\setQBFvariableslevel {k+1}} \models R_{i,k} \limply  tr_{k+1}(\psi) \\
    & \iff 
    \valuationdescr{\st}{\setQBFvariableslevel {k}} \models \forall \setQBFvariableslevel {k+1}, R_{i,k} \limply  tr_{k+1}(\psi) \\
    & \iff
    \valuationdescr{\st}{\setQBFvariableslevel k} \models tr_k(\impbel i \alpha)
\end{align*}

\item
    Let $\phi = \impbelmost i \phi$.
    The proof is analogous to that for operator $\impbel{i}$ above.
    In particular, we use relation $\relepistcomp{i}$ and formula $\lnot R_{i,k}$.
\end{itemize}

Therefore, $(\st_0,\setbelbasecustom) \models \phi_0$ iff  $\valuationdescr{\st_0}{\setQBFvariableslevel 0} \models tr_0(\phi)$.
In addition,
$\valuationdescr{\st_0}{\setQBFvariableslevel 0} \models tr_0(\phi)$ is equivalent to $\exists \setQBFvariableslevel0 (
            \descr{\st_0}{\setQBFvariableslevel0} \land tr_0(\phi_0)
        )$ is QBF-true.
This concludes the proof.\qed
\end{proof}
} 

\inlongversiononly{%
In  \cite{DBLP:journals/corr/abs-1907-09114}, it is proved that the previous model checking problem formulated in the belief base semantics is PSPACE-hard,
already for the fragment of $\lang$ with only ``at least'' implicit belief operators, but with no ``at most'' implicit belief operators involved. 
Thus, the fact that the generic model checking  problem given in Figure~\ref{figure:genericmc} runs in polynomial space as well as Proposition~\ref{prop:tqbftrans} allow us to state the following complexity result.%
} 

\begin{theorem}%
\label{theo:comp}
Model checking $\lang$-formulas is PSPACE-complete. 
\end{theorem}

\section{Implementation and experimental results}%
\label{sec:expres}

We implemented a symbolic model checker,%
\footnote{Available at {https://src.koda.cnrs.fr/tiago.de.lima/lda/}}
which uses the translation to TQBF.
The resulting TQBF is then translated into a binary decision diagram (BDD), in the same way as done in \cite{DBLP:journals/logcom/BenthemEGS18}.
The program is implemented in Haskell and the BDD library used is {HasCacBDD} \cite{hascacbdd}.
It was compiled with GHC 9.2.7
in a MacBook Air with a 1.6 GHz Dual-Core Intel Core i5 processor and 16 GB of RAM,
running macOS Ventura 13.3.1.

\inJELIAversiononly{%
We formalized an example in which agents are members of a selection committee.
Votes are secret but the result of the selection is public.
The logic is used to compute what the agents can deduce from the public information given their private beliefs. 
Two variants of the example are considered,
one involving agents' first-order beliefs
and a second one involving higher-order beliefs.%
} 



%

\begin{table}[ht]
\begin{center}
    \begin{tabular}{l@{\quad}r@{\quad}r@{\quad}r@{\quad}r@{\quad}r@{\quad}r@{\quad}r@{\quad}r}
        \hline
        $cands = voters =|\AGT|$             & $3$       &  $4$      &    $5$     &        $6$ &        $7$ &        $8$ &     $9$ &  $10$ \\
        $|\atmset|$ & $9$       & $16$      &   $25$     &       $36$ &       $49$ &       $64$ &    $81$ & $100$\\
        $ratoms$              & $100$     &$164$      &  $244$     &      $340$ &      $452$ &      $580$ &      $724$ & $884$\\
        $states$             & $2^{309}$ & $2^{672}$ & $2^{1245}$ & $2^{2076}$ & $2^{3213}$ & $2^{4704}$ & $2^{6597}$ & $2^{8940}$\\
        \hline
        Execution time (sec.)    &   $0.076$ &   $0.015$ &    $0.026$ &    $0.047$ &    $0.066$ &    $0.101$ & $0.157$ & $0.248$\\
        \hline
    \end{tabular}
    \bigskip

%

    \begin{tabular}{l@{\quad}r@{\quad}r@{\quad}r@{\quad}r@{\quad}r@{\quad}r@{\quad}r@{\quad}r}
        \hline
        $cands = voters =|\AGT|$            &       $3$ &       $4$ &        $5$ &        $6$ &        $7$ &        $8$ &        $9$ &        $10$\\
        $|\atmset|$ &       $9$ &      $16$ &       $25$ &       $36$ &       $49$ &       $64$ &       $81$ &       $100$\\
        $ratoms$              &     $133$ &     $210$ &      $305$ &      $418$ &      $549$ &      $698$ &      $865$ &      $1050$\\
        $states$             & $2^{408}$ & $2^{856}$ & $2^{1550}$ & $2^{2544}$ & $2^{3892}$ & $2^{5648}$ & $2^{7866}$ & $2^{10600}$\\
        \hline
        Execution time (sec.)    &   $0.081$ &   $0.063$ &    $0.334$ &    $3.066$ &   $17.588$ &   $90.809$ &         KO &          KO\\
        \hline
    \end{tabular}
\end{center}
\caption{%
    \iftoggle{long}{%
        \label{tab:experiments}Symbolic model checker performance on Examples \ref{eg:committee} (above) and \ref{eg:committee-variant} (below).%
    }{%
        \label{tab:experiments}Symbolic model checker performance on two examples.%
    }}
\end{table}

\inJELIAversiononly{Table \ref{tab:experiments} shows the performance of the model checker on the two variants of the example.}
\inlongversiononly{Table \ref{tab:experiments} shows the performance of the model checker on the examples of Section~\ref{sec:example}.}
It shows execution times for different instances.
For both examples, the size of the model ($states$) is given by the number of possible valuations times the number of possible multi-agent belief bases:
$2^{|\atmset|} \times (2^{ratoms})^{|\AGT|}$.
The value of $ratoms$ is the number of ``relevant atoms''.
There is one such atom
for each formula in $\setrelevantformulas$,
each propositional variable appearing in $\setrelevantformulas$ and in the input formula,
each formula $\alpha$ that is a sub-formula of the input formula,
plus one atom for each formula $\triangle_i\alpha$ such that $\alpha \in \setrelevantformulas$.
The number of states gives an idea of the size of the search space for modal formulas.
In principle, to check a formula of the form $\Box^o \phi$, one must check $\phi$ in every state of the model.
Because of that, a naive implementation cannot be used.
Indeed, in our tests with such a solution, no instance could be solved under the timeout of 10 minutes.

\iftoggle{long}{%
One can notice that the model checker is slower
in the case of $3$ candidates than in the case of $4$
candidates (and in Example~\ref{eg:committee}
the latter is true even up to $7$ candidates). 
The reason is that the input formula is true for $3$ candidates, whereas it is false on all the other cases.
Checking that a box formula is false is easier, because the checker needs to find only one state where the formula in the scope of the box operator is false.
Also note that instances of Example~\ref{eg:committee} are solved much faster than those of Example~\ref{eg:committee-variant}.
This is due to two factors.
First, Example~\ref{eg:committee-variant} has larger belief bases, which imply a larger number of states.
Second, the input formula of the second example has a larger modal depth, which obliges the checker to generate a larger search tree.%
}{
One can notice that the model checker is slower
in the case of $3$ candidates than in the case of $4$
candidates (and in the first example
the latter is true even up to $7$ candidates). 
The reason is that the input formula is true for $3$ candidates, whereas it is false on all the other cases.
Checking that a box formula is false is easier, because the checker needs to find only one state where the formula in the scope of the box operator is false.
Also note that instances of the first example are solved much faster than those of the second.
This is due to two factors.
First, the second example has larger belief bases, which imply a larger number of states.
Second, the input formula of the second example has a larger modal depth, which obliges the checker to generate a larger search tree.%
} 

\inlongversiononly{%
\section{Dynamic extension}%
\label{sec:dynamic}

In this section,
we present a simple extension of 
the language $ \lang$
by dynamic operators
for modeling the agents' belief dynamics
of private type. Similar operators
were introduced in \cite{LoriniAI2020}.
The novel result
of this section 
is to show that adding them to the language
$ \lang$ 
does not increase complexity
of the model checking problem. 
More generally, we provide a simple
dynamic
extension
of the static language of 
implicitly 
believing \emph{at least}
and implicitly  believing \emph{at most}
whose model checking problem remains in PSPACE. 

The extended language
 is noted $ \langplus$ and is defined by the following grammar:
 \begin{center}\begin{tabular}{llcl}
 $\langplus \defin $ & $\phi$  & $\bnf$ & $\alpha  \mid \neg\phi \mid \phi  \wedge \phi   \mid  \impbel{i}\varphi \mid \impbelmost{i}\varphi
 \mid [+_i \alpha ]\varphi,
                        $\
\end{tabular}\end{center} 
where $\alpha$ ranges over $\fraglang$
and $i $ ranges over $\AGT $.

Events of
type
$+_i \alpha$
are called   informative events. 
In particular,
$+_i \alpha$
is the event of 
agent $i$
privately expanding its belief
base with $\alpha$. 

The formula $[+_i \alpha]\varphi$ is read 
``$\varphi$ holds after the 
informative 
event $+_i \alpha $ has occurred''.
It has  the following semantic
interpretation relative to a model. 
\begin{definition}[Satisfaction relation, cont.]\label{truthcond3}
Let $\st= ( \belbaseset_1,\ldots,\belbaseset_n, \stateval)\in 
\setbelbase$ and let $ (\st ,\iconstraint)   \in\classbelbase$. Then:
\begin{eqnarray*}
 (\st ,\iconstraint) \models [+_i \alpha] \varphi & \Longleftrightarrow & 
  (\st^{+_i \alpha} ,\iconstraint)\models\varphi ,
\end{eqnarray*}
where  $ \belbaseset_i^{+_i \alpha}= \belbaseset_i\cup \{\alpha\}
  \text{ and }
  \belbaseset_j^{+_i \alpha}= \belbaseset_j
  \text{ for all }j \neq i$.
\end{definition}
Intuitively speaking, the private belief expansion of $i$'s belief base by $\alpha$ simply consists in agent $i$
adding the information that $\alpha$ to its  belief base,
while all other agents keep their belief bases unchanged.
Let us go back to the Example 
\ref{eg:committee} we introduced in Section 
\ref{sec:example}
to illustrate the expressiveness
of our dynamic extension.
\begin{example}
It is worth noting that private belief dynamics
allow agents
to gather  new information 
and to gain new knowledge.
Suppose 
in the three-agent variant of the example
agent $2$
and agent $3$
privately learn
that $1$
voted for $2$.
This ensures that there is no longer any information asymmetry between agent $1$ and agents $2 $ and $3$.
Formally, we have 
\begin{equation*}
    (\st_0,\setbelbasecustom) \models
    \big[+_2\mathsf{vote}(1{,}2 )\big]\big[+_3 \mathsf{vote}(1{,}2 )\big] 
    \chi_0, 
\end{equation*}
where 
\begin{align*}
    \chi_0 \defin &  \bigwedge_{i \in  \{1, 2,3\}} \impbelonly{i}\psi_1,
                 \end{align*}
                 and 
                 $\psi_1$, $\st_0 $, $\setbelbasecustom$
                        are defined as  in Example \ref{eg:committee}
                        in Section 
\ref{sec:example}. 
This means that,
in the three-agent variant of the example,
after agent $2$ and agent $3$
privately learn that agent $1$
voted for $2$,
everybody 
only knows for whom
an agent 
voted and for whom 
she/he did not vote.
\end{example}

As the following proposition highlights,  
we have reduction principles for
the dynamic operators. 
\begin{proposition}\label{theo:reductionprinciples}
The
following equivalences are
valid for the class $\classbelbase$:
 \begin{align*}
&     [+_i\alpha] p \leftrightarrow  p \\
&    [+_i \alpha] \neg \varphi  \leftrightarrow   \neg [+_i\alpha]  \varphi   \\
&    [+_i  \alpha] (\varphi_1 \wedge \varphi_2) \leftrightarrow \big(  [+_i  \alpha] \varphi_1 \wedge
 [+_i \alpha] \varphi_2 \big) \\
    &    [+_i \alpha] \expbel{j}\beta \leftrightarrow  \expbel{j}\beta \text{ if } i \neq j \text{ or } \alpha \neq \beta \\
      &    [+_i \alpha] \expbel{i}\alpha \leftrightarrow \top \\
  &   [+_i  \alpha] \impbel{j}\varphi \leftrightarrow \impbel{j}\varphi   \text{ if } i \neq j  \\
    &   [+_i \alpha] \impbel{i}\varphi \leftrightarrow \impbel{i} (\alpha \rightarrow \varphi  )\\
  &   [+_i  \alpha] \impbelmost{j}\varphi \leftrightarrow \impbelmost{j}\varphi   \text{ if } i \neq j  \\
    &   [+_i \alpha] \impbelmost{i}\varphi \leftrightarrow 
    \big( \impbel{i}(\neg \alpha \rightarrow \varphi) \wedge \impbelmost{i}\varphi \big)
\end{align*}
\end{proposition}
\begin{proof}
We only prove  cases 
$\impbel{i}\varphi$
and
$\impbelmost{i}\varphi$, 
since  other cases are straightforward. 
\begin{eqnarray*}
  (\st ,\iconstraint)\models 
  [+_i \alpha] \impbel{i}\varphi
  & \Longleftrightarrow &
  (\st^{+_i \alpha} ,\iconstraint) \models 
   \impbel{i}\varphi,\\
   & \Longleftrightarrow &
\forall 
\st' \in  \iconstraint:
 \text{if } \st^{+_i \alpha} \relstate{i} \st'
      \text{ then }  ( \st' , \iconstraint) \models \varphi,\\
         & \Longleftrightarrow &
\forall 
\st' \in  \iconstraint:
 \text{if } \st \relstate{i} \st'
 \text{ and }
 \st' \models \alpha 
      \text{ then }  ( \st' , \iconstraint) \models \varphi,\\
         & \Longleftrightarrow &
 ( \st , \iconstraint) \models \impbel{i} (\alpha \rightarrow \varphi  ) .
\end{eqnarray*}
\begin{eqnarray*}
  (\st ,\iconstraint)\models 
  [+_i \alpha] \impbelmost{i}\varphi
  & \Longleftrightarrow &
  (\st^{+_i \alpha} ,\iconstraint) \models 
   \impbelmost{i}\varphi,\\
   & \Longleftrightarrow &
\forall 
\st' \in  \iconstraint:
 \text{if } \st^{+_i \alpha}\relstatecomp{i} \st'
      \text{ then }  ( \st' , \iconstraint) \models \varphi,\\
         & \Longleftrightarrow &
\forall 
\st' \in  \iconstraint:
 \text{if } \st \relstatecomp{i} \st'
 \text{ or }(
 \st \relstate{i} \st'
 \text{ and }
 \st' \models \neg \alpha )
      \text{ then } \\
     & &( \st' , \iconstraint) \models \varphi,\\
        & \Longleftrightarrow &
 ( \st , \iconstraint) \models      \impbel{i}(\neg \alpha \rightarrow \varphi) \wedge \impbelmost{i}\varphi  . 
  \ \ \ \ \ \ 
   \ \ \ \ \ \ 
    \ \ \ \ \ \ 
     \ \ \ \ \ \ 
      \ \ 
 \qed 
\end{eqnarray*}

\end{proof}

Model checking
for  
formulas in the language
$ \langplus$ is analogous to 
model checking
for
formulas
in 
$ \lang$
we defined in Section \ref{sec:mc}.
The  valid equivalences 
in Proposition 
\ref{theo:reductionprinciples}
could  be used
to find a procedure for reducing model checking
for  formulas
in 
$ \langplus$
to
model checking
for  formulas
in 
$ \lang$.
The problem is that such reduction
is exponential
due to the fact that every time
we find a formula
of type
$ [+_i \alpha] \impbelmost{i}\varphi$
we have to duplicate
it into two parts
$\impbel{i}(\neg \alpha \rightarrow \varphi) $
and $\impbelmost{i}\varphi$. 

Fortunately we can easily adapt the generic algorithm 
presented in Section \ref{sec:mc} in order to obtain a PSPACE procedure for model checking formulas
of the language $ \langplus$.
It  is sufficient to add the following case 
for the dynamic operators 
to the main routine of the algorithm in Figure  \ref{figure:genericmc}:
\begin{align*}
    \mathbf{case} \   [+_i \alpha]\psi : 	\mathbf{return} \ \mc {\st^{+_i \alpha}} \Gamma \psi  
\end{align*}
The resulting algorithm clearly runs in polynomial space.
 Thus, we can generalize
 the complexity result
 given in Theorem
 \ref{theo:comp}
 to the language $\langplus $.
\begin{theorem}
Model checking $\langplus$-formulas
is PSPACE-complete. 
\end{theorem}
} 

\section{Conclusion}%
\label{sec:conclusion}

This paper describes optimal procedures for model checking multi-agent only believing formulas.  
As far as we know, we are the first to tackle the problem of automating model checking for the logic of multi-agent only believing or knowing.
We implemented these procedures and presented some experimental results on computation time. 
\iftoggle{long}{%
Moreover, we extended the formalism with private belief expansion operators and showed that model checking remains PSPACE-complete. 
In the future, we plan to implement the dynamic extension
presented in Section \ref{sec:dynamic}
and to extend the setting to introspective agents whose logic of belief (resp.\ knowledge) is K45 (resp.\ S5).
Last but not least, we
}{
We
}
intend to apply our semantics for multi-agent only believing and model checking approach to epistemic planning. 
We believe that the compactness of our semantics can offer an advantage in terms of ease of implementation compared to the multi-relational Kripke semantics traditionally used in the context of epistemic planning
\cite{BolanderAndersen,DBLP:conf/ijcai/BolanderJS15}.
 
%
%
%
\bibliographystyle{splncs04}
\bibliography{biblio,biblio2}
\end{document}